\documentclass{fairmeta}
\usepackage[most]{tcolorbox}
\usepackage{amsmath}
\usepackage{amssymb}
\usepackage{amsthm}
\usepackage{booktabs}
\usepackage{float}
\usepackage{textcomp}
\usepackage{multirow}
\usepackage{graphicx}
\usepackage{tabularx}
\usepackage{microtype}

\newtheorem{theorem}{Theorem}
\title{APEX: Adaptive Policy Execution for\\ Precise Manipulation}

\author[1,*]{Mengfei Zhao}
\author[1,*]{Chenxi Jiang}
\author[1]{Tuo An}
\author[1,\dagger]{Jindou Jia}
\author[1,\dagger]{Jianfei Yang}

\affiliation[1]{MARS Lab, Nanyang Technological University}
\contribution[*]{Equal contribution}
\contribution[\dagger]{Corresponding authors}

\abstract{
Modern imitation learning methods, including visuomotor and Vision-Language-Action (VLA) policies, typically output high-level action references that are executed by low-level controllers. However, the absence of higher-order reference signals, together with the policy's lack of awareness of the underlying low-level control dynamics during training, inevitably induces an execution gap. As a result, realized actions deviate systematically from policy-commanded ones, with a critical impact on precision-sensitive manipulation. Prior work either modifies the policy architecture or the low-level controller, both requiring intrusive changes to the pretrained policy or packaged controller. This raises a natural question: when the policy and controller are both treated as inaccessible black boxes, can we bridge the execution gap? We propose \textbf{A}daptive \textbf{P}olicy \textbf{EX}ecution (\textbf{APEX}), a plug-and-play framework inserted between the policy and the controller that reconstructs a dynamically feasible reference from policy outputs and adapts at test-time according to low-level state feedback, with a provable convergence guarantee. Extensive empirical studies show that APEX reduces controller-induced tracking error by 41.2\% on demonstration replay and improves manipulation success by 4.8--25.8 percentage points across four visuomotor and VLA policy classes.
}

\metadata[Project site]{\url{https://murphyzhao04.github.io/apex-project-page/}}

\correspondence{
\email{jindou.jia@ntu.edu.sg}, \email{jianfei.yang@ntu.edu.sg}
}

\begin{document}
\maketitle
\section{Introduction}

\begin{figure}[t]
    \centering
    \includegraphics[trim=0.0cm 0.0cm 0.0cm 0.0cm, clip, width=\linewidth]{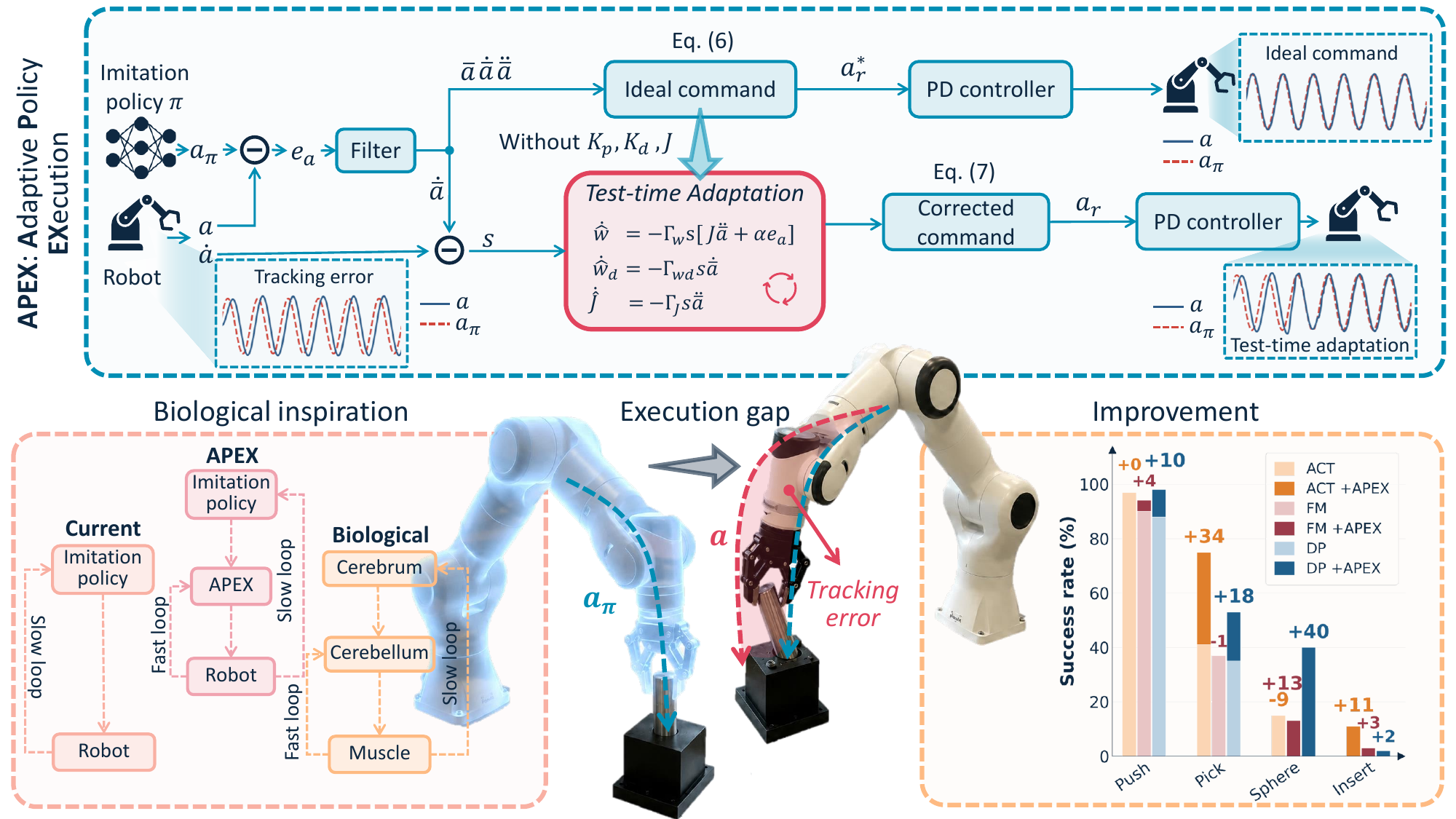}
    \vspace{-1em}
    \caption{\textbf{Overview of APEX.} \textbf{Bottom center:}
    At deployment, the command is issued by a learned policy, and the motion actually realized by the PD controller can diverge from the commanded reference.
    \textbf{Bottom left:} Motivated by biological motor control, where fast
    sensorimotor feedback loops correct small deviations on the fly rather
    than regenerating the time-consuming high-level plan, APEX inserts an \emph{online
    correction layer} between the policy and controller, leaving both untouched.
    \textbf{Top:} APEX first reconstructs a dynamically
    feasible reference from the policy output through an adaptive passive filter,
    supplying the higher-order signals the controller requires.
    Second, it augments this reference with test-time adaptive terms whose coefficients are updated online from the real-time tracking error, compensating for unknown system dynamics and controller parameters that are typically inaccessible on commercial platforms.    \textbf{Bottom right:} Across four policy classes, APEX reduces
    tracking error and improves the success rate.}
    \label{fig:teaser}
\vspace{-1.5em}
\end{figure}

Imitation learning, which trains robot policies from demonstration trajectories~\citep{pomerleau1988alvinn,
argall2009survey}, has become a cornerstone of modern robot learning, powering recent advances in visuomotor        
policies~\citep{zhao2023learning, chi2023diffusion, jia2026action, jia2026mars} and Vision-Language-Action (VLA) models~\citep{kim2024openvla,  
mees2024octo, black2024pi0, pertsch2025pi05}. Such schemes typically produce joint targets~\citep{zhao2023learning} or end-effector~\citep{chi2023diffusion,padalkar2023open,kim2024openvla} trajectories~\citep{padalkar2023open}, which are subsequently executed by a low-level controller, such as a Proportional-Derivative (PD) scheme. However, inadequate high-level outputs (e.g., discontinuities and missing higher-order references) and imperfect low-level controllers (e.g., unmodeled dynamics and parameter uncertainties) make precise execution elusive in real-world deployment.
We refer to this discrepancy between policy-level command and robot-native execution as the \textit{test-time execution gap}. The issue becomes particularly critical in precision-sensitive manipulation tasks, where even small tracking errors can disrupt contact timing or cause misalignment during grasping and insertion, ultimately leading to task failure.

Existing approaches attempt to mitigate this problem from several directions. One line of work augments policy      
outputs with richer trajectory-level representations, for example, through spline-based parameterizations or 
explicit velocity augmentations that provide analytical velocity profiles for low-level
tracking~\citep{yang2026abpolicy, hechtl2026tracking, bai2026flash}. However, these methods are inapplicable when the underlying policy architecture and training pipeline cannot be modified. Another line learns residual corrections on top of a base    
policy or controller~\citep{silver2018residual, johannink2019residual}. Although effective, these methods require additional data collection and a separate training phase. More recently, test-time refinement methods have been proposed to smooth or refine policy outputs without retraining~\citep{zhao2023learning, black2025rtc, liu2024bid, lipo2025, leave2025observation}. These approaches mainly address policy-side artifacts, such as latency and temporal discontinuities, but have not considered the tracking errors inherent to the low-level controller. This raises a natural question: can the execution gap be bridged without modifying either the policy or the controller?


A useful perspective comes from biological motor control. When humans perform fine manipulation, e.g., threading a needle or placing a cup on a saucer, the high-level plan is rarely regenerated during movement, as doing so would be time-consuming. Instead, the nervous system continuously adjusts the ongoing movement through fast sensorimotor feedback loops that correct small deviations on the fly~\citep{aoki2019universal, markov2021cerebellar, jia2025feedback}. This separation between what to do (slow, policy-level) and how to execute it accurately (fast, execution-level) suggests that the test-time execution gap is more naturally closed by an online correction layer acting on the reference stream, rather than by modifying the policy and the controller themselves.

Building on this principle, we propose the \textbf{A}daptive \textbf{P}olicy \textbf{EX}ecution (\textbf{APEX}) framework, which implements such correction in two stages. First, APEX reconstructs a smooth, dynamically feasible reference from policy outputs through an adaptive passive filter~\citep{8735721}, supplying the higher-order signals that low-level controllers require. Second, it augments this reference with test-time adaptive terms that compensate for unknown system dynamics and controller parameters, with coefficients updated online from the real-time tracking error. Notably, APEX is agnostic to the specific policy and controller class, also admitting provable convergence guarantees.

We evaluate APEX in three settings. First, we remove the policy and replay demonstration trajectories through the low-level controller, isolating the execution gap from policy performance. APEX reduces 41.2\% of the tracking error. Second, we pair APEX with four policy classes, including ACT~\citep{zhao2023learning}, diffusion policy~\citep{chi2023diffusion}, flow matching~\citep{flowmatching}, and the VLA $\pi_{0.5}$~\citep{pertsch2025pi05}, across multiple manipulation tasks and controller configurations, covering both deterministic and generative paradigms. APEX improves the success rate by 4.8--25.8 percentage points (pp), with clear gains on PegInsertion, a task with only 3 mm radial clearance. Finally, we deploy APEX on a UR5 for precision-sensitive tasks. The most pronounced gain occurs on the contact-rich PegInsertion task, where APEX improves the flow matching policy's success rate from 20\% to 35\%.






\section{Related Work}
\label{sec:related}

We discuss three lines of work most closely related to the test-time execution gap.

\textbf{Modifying the action representation.}
One line of work modifies how policies generate actions. The action space itself, joint angles, end-effector poses, or variable-impedance commands, shapes learning behavior~\citep{martin2019variable, aljalbout2024role}. Various methods enrich the output with B-spline trajectories or velocity augmentations to supply higher-order tracking signals~\citep{yang2026abpolicy, hechtl2026tracking, bai2026flash}, or output torques directly~\citep{kumar2023decap}. All require retraining, which is prohibitive for large pretrained VLAs with action spaces fixed by collected datasets~\citep{padalkar2023open}. APEX leaves the policy untouched and corrects execution downstream.

\textbf{Learning residual corrections.}
Residual learning keeps a base policy or controller fixed and learns an additive correction~\citep{silver2018residual, johannink2019residual}. Recent extensions train the residual via differentiable simulation for legged locomotion~\citep{luo2024residual}, learn residual dynamics online within model predictive control for contact-rich manipulation~\citep{huang2024adaptive}, or fit a delta action model to bridge sim-to-real gaps for humanoids~\citep{he2025asap}. All require additional training or interactive data collection. APEX requires no additional data collection. It updates existing unknown parameters at test-time from the real-time tracking error.

\textbf{Refining policy outputs at test time.}
Another line refines policy outputs at test time without retraining, mostly targeting action-chunking policies like temporal ensembling~\citep{zhao2023learning}, bidirectional decoding~\citep{liu2024bid}, real-time chunking via inpainting~\citep{black2025rtc}, jerk-minimizing post-optimization~\citep{lipo2025}, and observation-conditioned asynchronous correction~\citep{leave2025observation}. These methods reshape the action stream to address policy-side artifacts such as latency and stale predictions, implicitly assuming the resulting reference is tracked faithfully on hardware. We show this assumption breaks down in Sec.~\ref{sec:replay}. Even a replayed expert reference produces systematic tracking error under a commercial off-the-shelf controller. In contrast, APEX explicitly accounts for execution-side errors.

\section{Problem Formulation}
Modern robot policies are often trained by imitation learning from demonstration trajectories~\citep{pomerleau1988alvinn,
argall2009survey}. Given a demonstration dataset $\mathcal{D} = \{(o_i, a_i)\}_{i=1}^N$ of observation-action pairs with $N$ samples, a policy $\pi$ is trained as a mapping
\begin{equation}
    \pi: o \mapsto a_\pi,
    \label{eq:policy}
\end{equation}
where $o$ denotes the current sensor observation and $a_\pi$ represents a policy-level action such as joint-position targets~\citep{zhao2023learning} or end-effector pose/delta-pose targets~\citep{chi2023diffusion,padalkar2023open,kim2024openvla}. At deployment, the policy produces a time-varying action reference $a_\pi(t)$, which is executed by a low-level controller. Such a controller is often inaccessible to users, particularly on commercial platforms such as UR and xArm.



Let $a(t)$ denote the actual action executed by the controller. The low-level dynamics follow the general \textit{Lagrange} equation with inertia matrix $M$, Coriolis matrix $C$, gravity vector $G$, and control torque $\tau$. Most policy-based architectures adopt a low-level PD controller $\tau_c$ with gains $K_p$ and $K_d$ to track $a_\pi$. Then we can formalize them as
\begin{equation}
    M(a)\ddot{a} + C(a, \dot{a})\dot{a} + G(a) = \tau,
    \quad \tau_c = K_p(a_\pi - a) - K_d\dot{a}.
    \label{eq:vector_dynamics}
\end{equation}
Define the tracking error $e_a = a_\pi - a$. In the absence of desired velocity and acceleration signals of $a_\pi(t)$, as well as the model knowledge, \textit{Lyapunov}'s theorem implies that the above PD controller inevitably incurs noticeable tracking errors~\citep{slotine1991applied}. Moreover, the PD controller is often inaccessible to users, especially on commercial platforms such as UR and xArm. 

\paragraph{Our objective:}
In this work, our goal is to construct a \textit{test-time correction module} that refines the policy output $a_\pi$ into an adjusted command $a_r$, such that feeding $a_r$ into the low-level PD controller yields an accurate tracking of the original $a_\pi$, i.e., $e_a \to 0$. Crucially, to ensure practicality, it is expected to require no modification to either the high-level policy or the low-level controller.

\section{Adaptive Policy Execution}
\label{sec:method}

As shown in Fig.~\ref{fig:teaser}, APEX is a test-time execution adapter that transforms the intended action from the high-level policy into an adjusted command for the low-level execution controller. In this section, we first present APEX under \emph{known} execution dynamics to illustrate its core principle, and then extend the formulation to the \emph{unknown}-dynamics setting by introducing a test-time adaptation algorithm, for which we provide a stability analysis.

\subsection{Correction with Known Execution Dynamics}
\label{sec:ideal_compensation}

We begin with an idealized setting in which the low-level execution dynamics are known. Without loss of generality, we adopt the \textit{joint} coordinate as the action space, so that $a_\pi(t)$ denotes the policy-intended joint angle. With the fixed PD controller of the platform, the closed-loop dynamics of each joint take the form
\begin{equation}
    J\ddot a = K_p(a_r - a) - K_d\dot a + d,
    \label{eq:inner_dynamics}
\end{equation}
where $J$ denotes the inertia of the joint motor and $d(t)$ lumps all unmodeled 
effects, including joint coupling and other nonlinear terms. Compared with the full \textit{Lagrangian} model in Eq.~\eqref{eq:vector_dynamics}, Eq.~\eqref{eq:inner_dynamics} treats each joint as a perturbed single-input channel.

Considering that the final goal is to ensure that the realized motion follows $a_\pi(t)$ exactly, substituting $a = a_\pi$ into Eq.~\eqref{eq:inner_dynamics} and solving for $a_r$ would yield
\begin{equation}
    a_r = a_\pi + \frac{J}{K_p}\,\ddot{a}_\pi + \frac{K_d}{K_p}\,\dot{a}_\pi - \frac{d}{K_p}.
    \label{eq:ideal_ar}
\end{equation}
This expression reveals that faithfully performing $a_\pi$ requires acceleration- and velocity-dependent offsets in the command, which current policies do not provide~\citep{zhao2023learning,chi2023diffusion,pertsch2025pi05}. Eq.~\eqref{eq:ideal_ar} is therefore not directly applicable.

Although finite differencing is a practical remedy in implementation, we adopt an adaptive passive filter~\citep{8735721} here for theoretical completeness. We construct an auxiliary smooth reference $\bar a$ together with its derivatives, generated by the adaptive passive filter driven by the tracking error $e_a = a_\pi - a$
\begin{equation}
    \dot y = -K_1 y - K_1 e_a,
    \qquad
    \dot{\bar a} = -K_2 y,
    \label{eq:abar_filter}
\end{equation}
where $y$ is an internal filter state and $K_1, K_2 > 0$ are filter gains. The trajectory $\bar a$ is obtained by integrating $\dot{\bar a}$. Intuitively, whenever $e_a \neq 0$ the filter drives $\bar a$ toward $a_\pi$ in a smooth, bandwidth-limited manner~\citep{8735721}. At steady state ($e_a \to 0$), $\bar a$ coincides with $a_\pi$, and Eq.~\eqref{eq:abar_filter} simultaneously yields the well-defined $\dot{\bar a}$ and $\ddot{\bar a} = K_1 K_2 (y + e_a)$ that Eq.~\eqref{eq:ideal_ar} requires.

Replacing $a_\pi$ and its derivatives in Eq.~\eqref{eq:ideal_ar} with the filtered counterparts and adding a feedback term to anchor $\bar a$ to $a_\pi$ via $a$, we obtain the \textit{ideal command}
\begin{equation}
    a_r^\star
    =
    \bar a
    + \frac{J}{K_p}\,\ddot{\bar a}
    + \frac{K_d}{K_p}\,\dot{\bar a}
    + \frac{\alpha}{K_p}\,e_a,
    \label{eq:stabilized_ar}
\end{equation}
with a gain $\alpha > 0$. The last term steers the realized motion toward $a_\pi$ and stabilizes the closed loop (convergence is established later). Note the term $d(t)$ is dropped from the design, as obtaining its analytical form is typically nontrivial. We instead adopt the standard assumption in adaptive control~\citep{slotine1991applied} that $d(t)$ is bounded, i.e., $|d(t)| \le \bar d$ for some unknown $\bar d > 0$. 


\subsection{Correction under Unknown Execution Dynamics}
\label{sec:adaptive_design}

Building on Eq.~\eqref{eq:stabilized_ar}, a direct implementation would require knowledge of the parameters $J$, $K_p$, and $K_d$, which are typically inaccessible on closed-source commercial arms. APEX instead introduces a test-time adaptive algorithm~\citep{8735721} that learns these unknown parameters directly from execution feedback. Before proceeding, we re-parameterize Eq.~\eqref{eq:stabilized_ar} as
\begin{equation}
    a_r
    =
    \bar a
    +  w \bigl[ J\,\ddot{\bar a} + \alpha\,e_a\bigr]
    +  w_d\,\dot{\bar a},
    \label{eq:adaptive_ar}
\end{equation}
where $w$ plays the role of $1/K_p$ and $w_d$ corresponds to $K_d/K_p$. These scalar weights together with $J$ are \emph{unknown a priori}. APEX performs a test-time adaptive learning on these coefficients
\begin{equation}
    \dot{\hat w}   = -\Gamma_w\, s \bigl[\hat J\,\ddot{\bar a} + \alpha\,e_a\bigr],
    \qquad
    \dot{\hat w}_d = -\Gamma_{wd}\, s\,\dot{\bar a},
    \qquad
    \dot{\hat J}   = -\Gamma_{J}\, s\,\ddot{\bar a},
    \label{eq:adaptive_law}
\end{equation}
where $\Gamma_w, \Gamma_{wd}, \Gamma_{J} > 0$ are learning rates, $\hat{(\cdot
)}$ denotes the estimation of signal $(\cdot)$, and $s = \dot a - \dot{\bar a}$. Here, $\dot a$ is the action velocity measured from robot state feedback. For joint-space coordinates, this corresponds to the measured joint velocity. For task-space coordinates, it can be obtained from end-effector velocity feedback or computed from joint velocity through the robot \textit{Jacobian}. Importantly, APEX does not differentiate the policy output $a_\pi$. The reference velocity $\dot{\bar a}$ is generated internally by the reconstruction filter in Eq.~\eqref{eq:abar_filter}. 


Eqs.~\eqref{eq:abar_filter},~\eqref{eq:adaptive_ar}, and~\eqref{eq:adaptive_law} together specify the APEX framework. We now establish that the resulting closed-loop system, the joint dynamics~\eqref{eq:inner_dynamics} interconnected with the filter~\eqref{eq:abar_filter}, correction~\eqref{eq:adaptive_ar}, and adaptation laws~\eqref{eq:adaptive_law}, is stable and achieves asymptotic tracking.

\begin{theorem}
\label{thm:nominal_main}
Consider the execution model~\eqref{eq:inner_dynamics} with unknown parameters $J, K_p, K_d > 0$ and a bounded disturbance satisfying 
$|d(t)| \le \bar d$, driven by the corrected command~\eqref{eq:adaptive_ar} with the filtered reference~\eqref{eq:abar_filter} and adaptation law~\eqref{eq:adaptive_law}. Assume that the policy action $a_\pi$ is fixed for the nominal analysis. For any positive gains $K_1, K_2, \alpha, \Gamma_w, \Gamma_{wd}, \Gamma_{J} > 0$ and any finite initial condition, all closed-loop signals remain bounded, and tracking error $e_a$ converges to a residual set, i.e.,
\begin{equation}
        \limsup_{t\to\infty} |s(t)| \;\le\; c_1 \,\bar d ,
        \qquad
        \limsup_{t\to\infty} |e_a(t)| \;\le\; c_2 \, \bar d,
        \label{eq:main_result}
\end{equation}
for constants $c_1, c_2 > 0$ depending on the design gains.
\end{theorem}

\begin{proof}
See Appendix~\ref{sec:appendix_stability} for complete proof details.
\end{proof}

We emphasize that, although several assumptions are introduced for theoretical completeness, our empirical results below show that APEX achieves substantial gains even in regimes where such idealized conditions are unlikely to hold in reality.
\section{Experiments}
\label{sec:exp}


Our central hypothesis is that the test-time execution gap arises not only from high-level policy prediction errors but also from suboptimal execution of the low-level controller. To verify, we design the experiments around three progressively broader questions: (1) How much do low-level controllers contribute to test-time execution failures? (2) Does APEX effectively bridge the joint policy--controller execution gap? (3) If so, does this benefit generalize across policies, action representations, simulators, and platforms? We describe the experimental setups in Appendix~\ref{sec:setup}.


\subsection{Imprecise controller alone can cause significant failures}
\label{sec:replay}

\begin{figure}
    \centering
    \includegraphics[width=\linewidth]{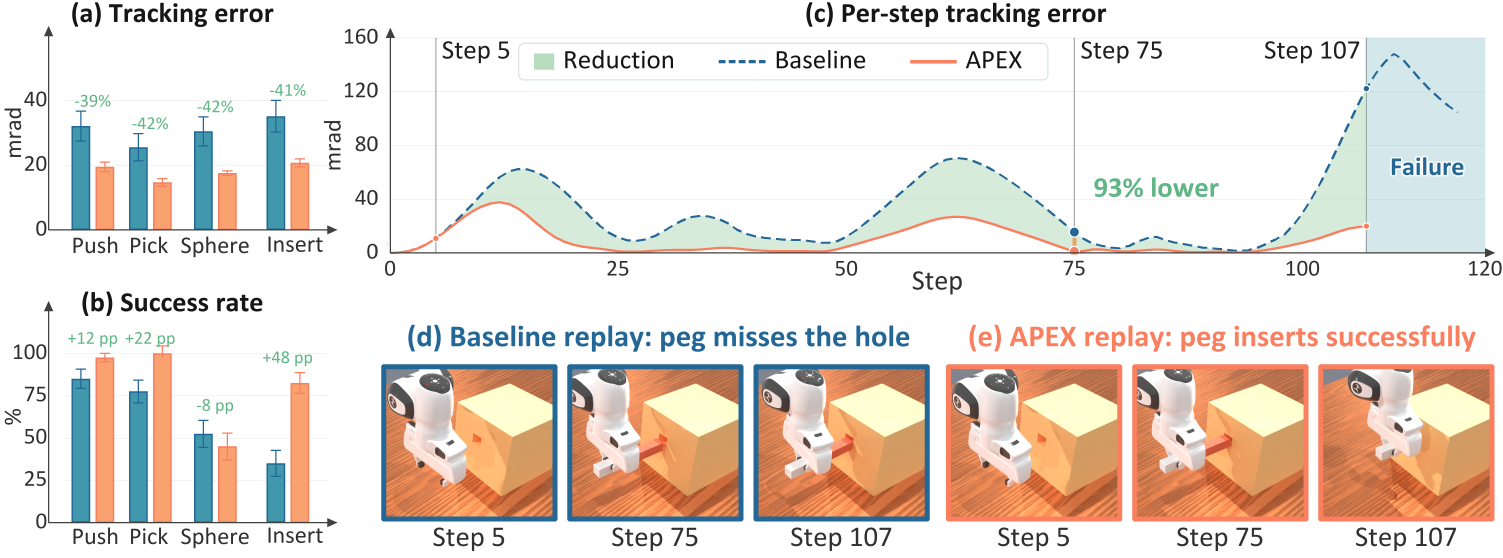}
    \vspace{-10pt}
    \caption{\textbf{Expert-replay benchmark in simulation.} \textbf{(a)} Under a vanilla PD controller, APEX reduces tracking error across all four tasks, leading to improved success rates in \textbf{(b)}. \textbf{(c)} On PegInsertion, APEX reduces the per-step tracking error throughout the trajectory, with up to a 93\% reduction near the insertion phase, enabling the successful insertion shown in \textbf{(d,e)}.}
    \label{fig:replay_combined}
    \vspace{-15pt}
\end{figure}

We first ask to what extent low-level controller execution errors degrade overall task performance. To answer this question, we consider an idealized setting: instead of using actions predicted by a learned policy, we directly replay expert demonstration trajectories as action references. These references represent best-case policy outputs. If execution failures are mainly due to inaccurate policy predictions, then this idealized setting should produce reliable rollouts.

Fig.~\ref{fig:replay_combined} shows that this is not the case. In the following experiments, \textbf{Baseline} denotes uncorrected execution, where the reference command is sent directly to the low-level execution stack without APEX. Under mismatched low-level execution parameters, the realized motion still deviates substantially from the expert reference. As shown in Fig.~\ref{fig:replay_combined}(a), baseline replay produces tracking errors of roughly 25--35 mrad across the four tasks. These deviations are large enough to affect task outcomes. Fig.~\ref{fig:replay_combined}(b) shows that replay success varies widely across tasks and drops sharply on precision-sensitive settings such as PegInsertion. Even with expert-level action references, an inaccurate low-level controller can fail to execute the intended motion accurately.

Without changing the replayed reference, APEX reduces the mean tracking error from 30.84 mrad to 18.15 mrad across all replay settings, a 41.2\% reduction. Replay success also increases from 62.5\% to 81.3\%. The PegInsertion tracking error in Fig.~\ref{fig:replay_combined}(c) shows APEX keeps the realized motion closer to the expert reference near insertion. The visual rollouts in Fig.~\ref{fig:replay_combined}(d,e) show the corresponding task outcome, where the baseline misses the hole while the APEX-corrected replay inserts successfully. 
The same effect holds on a physical UR5, as presented in Fig.~\ref{fig:replay_real_robot}. Across three real-world precision tasks, APEX reduces per-step tracking error and turns failed insertions into successful ones, confirming that the simulation findings generalize to hardware.

\begin{figure}[t]
    \centering
    \includegraphics[trim=0cm 0cm 0cm 5.8cm, clip, width=0.92\linewidth]{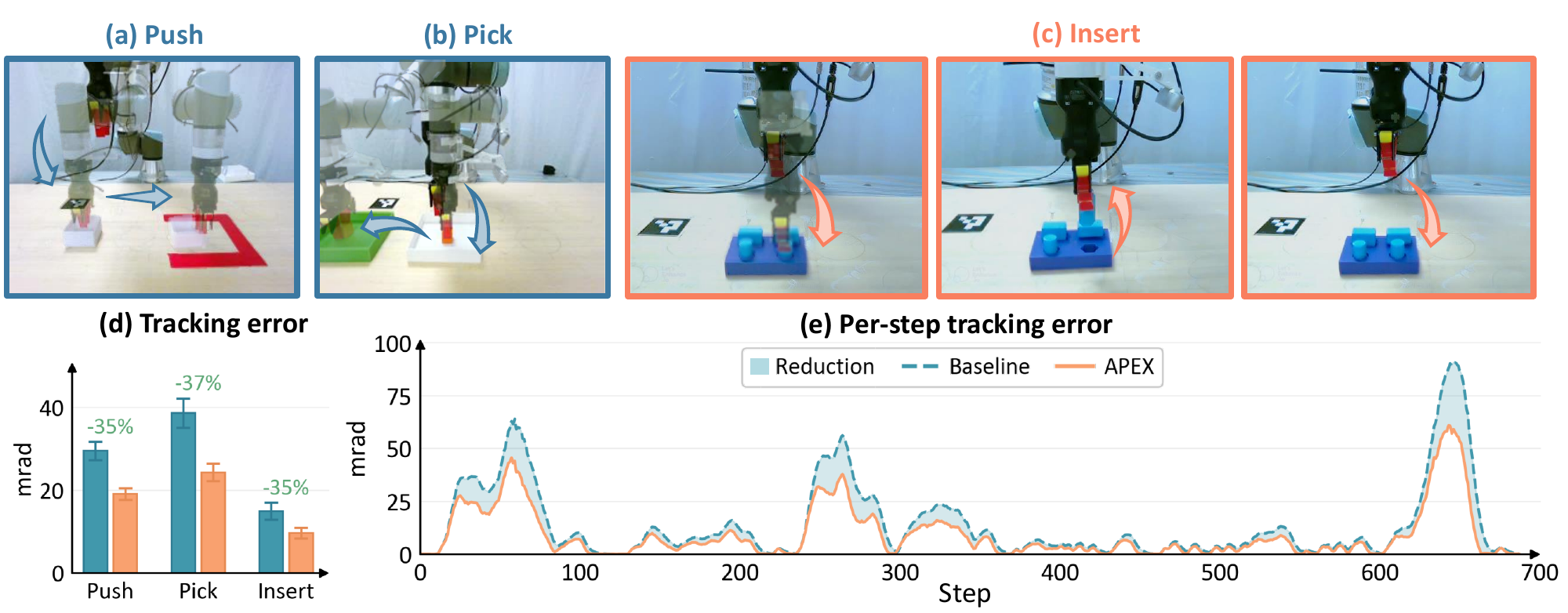}
    \vspace{-10pt}
    \caption{\textbf{Real-world deployment on a UR5.} APEX is evaluated on three precision-sensitive manipulation tasks executed on UR5:
    \textbf{(a)}~Push, \textbf{(b)}~Pick, and \textbf{(c)}~Insert, shown here
    as the recorded execution sequences. \textbf{(d)}~Across all three tasks,
    APEX reduces the tracking error under the robot's native position servo. \textbf{(e)}~On PegInsertion, APEX lowers the per-step tracking error. 
    }
    \label{fig:replay_real_robot}
    \vspace{-12pt}
\end{figure}

\subsection{APEX reliably bridges the policy--controller gap}
\label{sec:act_flagship}

\begin{figure}
  \centering
  \includegraphics[width=\linewidth]{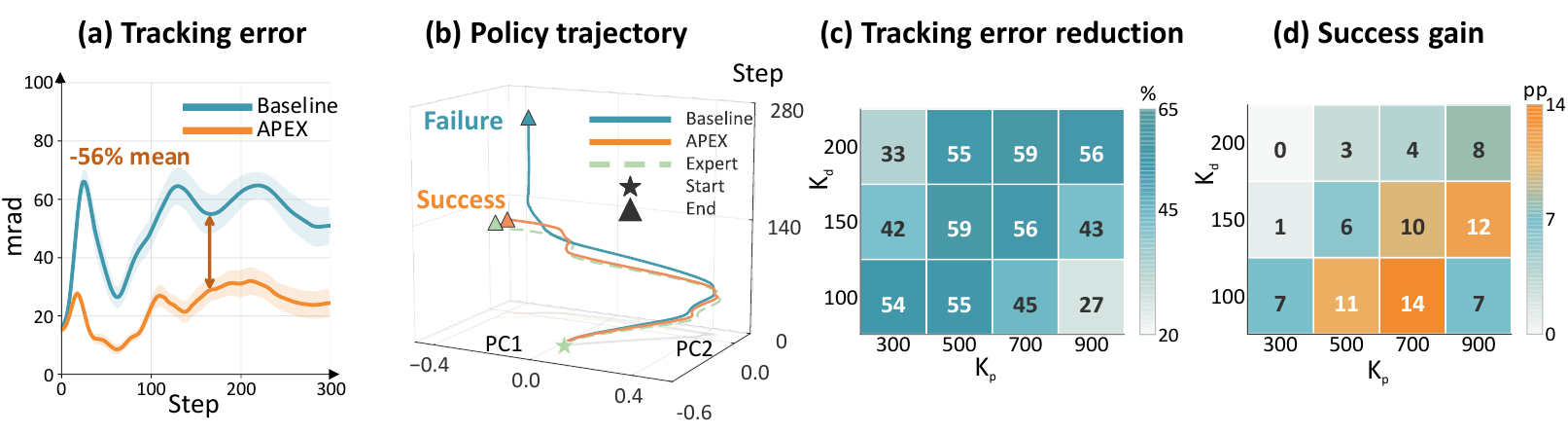}
  \vspace{-20pt}
  \caption{\textbf{Deployment with ACT policy on PegInsertion.} \textbf{(a)} Action tracking error. \textbf{(b)} PCA visualization of the joint-action trajectory, where PC1 and PC2 denote the first two principal components of joint positions, and steps are shown on the vertical axis. \textbf{(c)} Tracking-error reduction over the $K_p \times K_d$ grid. \textbf{(d)} Success-rate gain over the same grid.}
  \label{fig:act_execution_evidence}
  \vspace{-1em}
\end{figure}

In practice, no policy can output actions perfectly aligned with the demonstrations, which compounds the execution gap. We next evaluate whether APEX can reduce this action execution gap when the action reference is generated by a learned policy. Specifically, we use the ACT~\citep{zhao2023learning} as learned policy. In this setting, the tracking error is measured between the policy-intended action reference and the executed action. The results are provided in Fig.~\ref{fig:act_execution_evidence}. APEX substantially reduces tracking error from 51.55 mrad to 25.95 mrad across all ACT rollouts. Fig.~\ref{fig:act_execution_evidence}(a) and Fig.~\ref{fig:act_execution_evidence}(b) illustrate that the realized trajectory under APEX remains closer to the policy reference over time, whereas direct execution accumulates larger deviations.

We then examine whether this execution-level correction is robust across different low-level controller settings. Fig.~\ref{fig:act_execution_evidence}(c) reports the tracking-error reduction on PegInsertion over the evaluated \(K_p \times K_d\) settings. APEX reduces the policy-reference tracking error in every shown setting, indicating that the adapter is not tuned to a single PD configuration but remains effective across a range of inaccurate test-time controllers. The reduction in tracking error naturally leads to the success-rate gain, as shown in Fig.~\ref{fig:act_execution_evidence}(d). The gains are larger under more inaccurate control settings. Moreover, across all four ACT tasks and 48 controller settings, mean success increases from 27.4\% to 40.6\%. APEX improves 40 settings, leaves 5 unchanged, and degrades only 3. The full per-task grids are provided in Appendix~\ref{app:act_full_grid}.

\begin{table*}[t]
\centering
\small
\setlength{\tabcolsep}{5.2pt}
\renewcommand{\arraystretch}{1.15}
\caption{
Cross-policy success comparison under the same fixed controller mismatch. $\Delta$ reports the success-rate gain. All values are success rates in \%, and $\Delta$ is in percentage points (pp). 
}
\label{tab:cross_policy_success}
\vspace{0.5em}
\begin{tabular}{lccc ccc ccc}
\toprule
& \multicolumn{3}{c}{\textbf{ACT}}
& \multicolumn{3}{c}{\textbf{FM}}
& \multicolumn{3}{c}{\textbf{DP}} \\
\cmidrule(lr){2-4}
\cmidrule(lr){5-7}
\cmidrule(lr){8-10}
\textbf{Task}
& \textbf{Baseline} & \textbf{APEX} & \textbf{$\Delta$}
& \textbf{Baseline} & \textbf{APEX} & \textbf{$\Delta$}
& \textbf{Baseline} & \textbf{APEX} & \textbf{$\Delta$} \\
\midrule
PushCube-v1
& 97 & 97 & +0
& 90 & 94 & +4
& 88 & 98 & +10 \\

PickCube-v1
& 41 & 75 & +34
& 37 & 36 & -1
& 35 & 53 & +18 \\

PlaceSphere-v1
& 15 & 6  & -9
& 0  & 13 & +13
& 0  & 40 & +40 \\

PegInsertionSide-v1
& 0  & 11 & +11
& 0  & 3  & +3
& 0  & 2  & +2 \\
\bottomrule
\end{tabular}
\vspace{-1em}
\end{table*}

\subsection{APEX generalizes across policies and platforms}
\label{sec:cross_policy}

\begin{table}[t]
\centering
\small
\setlength{\tabcolsep}{4.5pt}
\renewcommand{\arraystretch}{1.05}
\caption{
Real-robot success rate  (\%) on a UR5 low-level controller.
$\Delta$ reports the success-rate change in pp. 
}
\label{tab:real_robot_success}
\begin{tabular}{lccc ccc}
\toprule
& \multicolumn{3}{c}{\textbf{FM}}
& \multicolumn{3}{c}{\textbf{DP}} \\
\cmidrule(lr){2-4}
\cmidrule(lr){5-7}
\textbf{Task}
& \textbf{Baseline} & \textbf{APEX} & \textbf{$\Delta$}
& \textbf{Baseline} & \textbf{APEX} & \textbf{$\Delta$} \\
\midrule
PushBox
& 55& 45& -10& 75& 80& +5\\
PickCube
& 55& 65& +10& 40& 45& +5\\
PegInsertion
& 20& 35& +15
& 30&  25& -5\\
\bottomrule
\end{tabular}
\end{table}

We next examine whether this benefit is tied to that particular policy-controller setting or whether it remains observable across different policy families, action representations, and benchmark platforms. We first evaluate ACT~\citep{zhao2023learning}, Diffusion Policy (DP)~\citep{chi2023diffusion}, and Flow Matching (FM)~\citep{flowmatching} on ManiSkill under the same controller. As shown in Tab.~\ref{tab:cross_policy_success}, APEX improves performance across all three policy classes, with average gains of
+9.0~pp on ACT, +4.8~pp on FM, and +17.5~pp on DP over the four tasks. These
results indicate that the benefit of execution-side correction is not specific to ACT.



We further test the same idea on a different benchmark platform, LIBERO Spatial, using the $\pi_{0.5}$~\citep{pertsch2025pi05}. In this setting, we change the simulator backend, task distribution, and action representation simultaneously. Unlike the ManiSkill policies, $\pi_{0.5}$ outputs end-effector delta actions, which are mapped to joint-space commands through inverse kinematics before execution correction is applied. Across the evaluated LIBERO tasks, APEX reduces the mean end-effector tracking error from 8.0 mm to 4.9 mm and improves the mean success rate from 72.1\% to 97.9\%. This result suggests that the observed benefit is not restricted to joint-position policies. Full results are provided in
Appendix~\ref{app:pi05_libero_kp}.

Finally, we validate APEX on a UR5 across three real-robot tasks,
where we collect demonstrations and train both DP and FM policies with
end-effector delta action outputs to assess transfer beyond simulation. As shown in Tab.~\ref{tab:real_robot_success}, APEX improves the
mean success rate for both policy classes, with the largest gain on the most
precision-sensitive task, PegInsertion, where it raises the FM success rate from
20\% to 35\% (+15 pp). 

\subsection{Ablation Study}
\label{sec:ablation}

\begin{figure}[t]
  \centering
  \includegraphics[width=0.82\linewidth]{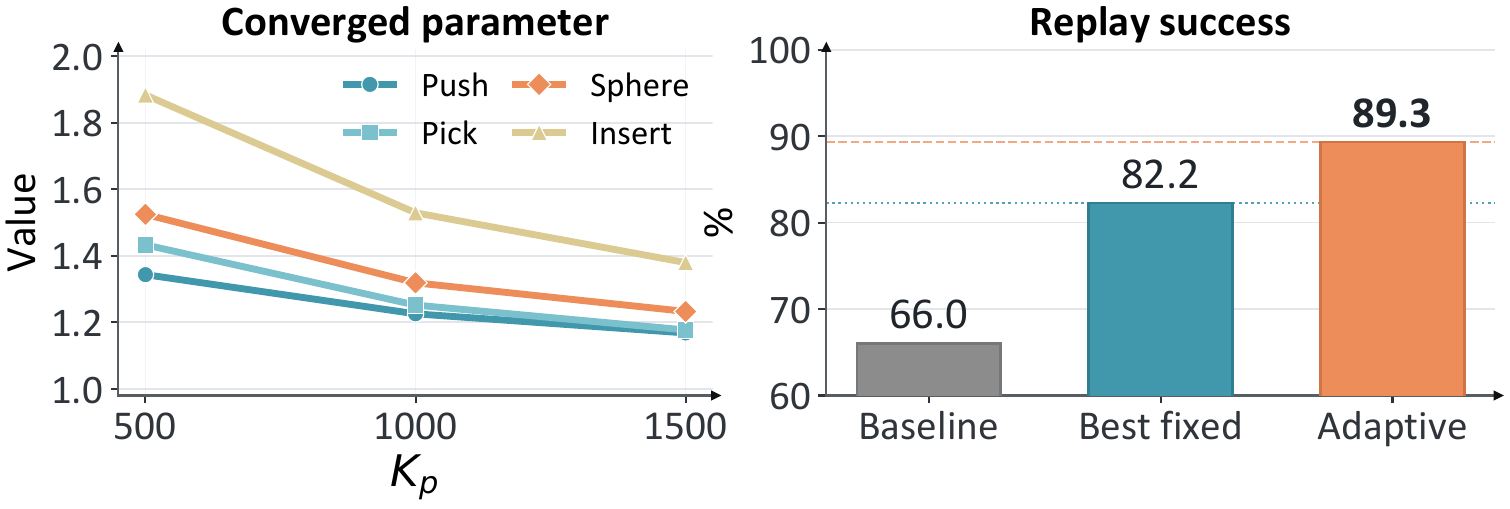}
  \caption{\textbf{Adaptive-update ablation.} \textbf{Left:} Convergence behavior under varying control gains. \textit{Value} denotes the effective multiplier on the position-error correction, corresponding to $1+w\alpha$ in Eq.~\eqref{eq:stabilized_ar}. 1 means no extra correction. It decreases with $K_p$, showing that APEX adapts a stronger correction when tracking is degraded. \textbf{Right:} Success rates for direct replay, the best fixed-parameter APEX, and adaptive APEX.}
  \label{fig:adaptation_evidence}
\end{figure}

This ablation asks whether the adaptive update in APEX is doing useful work or merely acting as a fixed offset. Further ablations on the adaptation gains and simulator-side joint dynamics parameters, including friction loss and armature, are deferred to the Appendix~\ref{app:physics-perturbation}.

We first perform a grid search over fixed correction coefficients with online adaptation disabled throughout execution (Appendix~\ref{app:fixed_vs_adaptive_sweep}, Tab.~\ref{tab:fixed_sweep_all_peg}). The best fixed-correction setting achieves a mean success rate of 82.2\% across the swept configurations. We compare this fixed-parameter variant against adaptive APEX in Fig.~\ref{fig:adaptation_evidence}. The adaptive variant reaches 89.3\%, surpassing both the uncorrected baseline and the best fixed-parameter variant. From an adaptive-control perspective, this finding is consistent with a classical insight in adaptive control. Control-oriented adaptation driven by closed-loop tracking error can outperform regression-oriented identification~\citep{richards2023control}.

\section{Limitations}
\label{sec:limitations}

One limitation of APEX is that it involves several adaptation gains, which currently require manual tuning. We plan to address this by integrating self-tuning mechanisms that adjust these gains online based on the observed tracking behavior. In addition, the disturbance term $d(t)$ in Eq.~\eqref{eq:inner_dynamics} is omitted in the current design, leading to a bounded convergence. In principle, $d(t)$ could be estimated online via a disturbance observer~\citep{jia2025foreseer} to reduce conservatism. 
We leave such an extension to future work.


\section{Conclusion}
\label{sec:conclusion}

We presented \textbf{APEX}, a test-time correction layer that bridges the \textit{execution gap} between high-level learned policies and low-level inaccessible controllers without modifying either. APEX reconstructs higher-order signals via a dynamic filter and refines the reference through an online adaptive law with stability guarantees. Empirically, APEX raises replay success from 62.5\% to 81.3\% and yields 4.8--25.8~pp improvements across four policy classes, including ACT, flow matching, diffusion policy, and $\pi_{0.5}$. These results suggest that even a well-trained policy can fail simply because its action references are not faithfully executed, and that a lightweight correction layer at the policy-controller interface can recover much of this lost performance.

\bibliographystyle{assets/plainnat}
\bibliography{paper}

\clearpage
\appendix

\section*{Appendices}
\addcontentsline{toc}{section}{Appendices} 

Within this supplementary material, we elaborate on the following aspects:
\vspace{1em}

\begin{itemize}
  \item \textbf{Appendix A:} Stability Analysis of Theorem~\ref{thm:nominal_main}
    \begin{itemize}
      \item \textbf{A.1:} Closed-Loop Error Dynamics
      \item \textbf{A.2:} Proof of Theorem~\ref{thm:nominal_main}
    \end{itemize}
    \item \textbf{Appendix B:} Experimental Setup
  \begin{itemize}
      \item \textbf{B.1:} Simulation Environments
      \item \textbf{B.2:} Real-robot Environments
      \item \textbf{B.3:} Evaluation Metrics
    \end{itemize}
  \item \textbf{Appendix C:} Implementation Details of APEX
    \begin{itemize}
        \item \textbf{C.1:} Hyperparameters
    \end{itemize}
  \item \textbf{Appendix D:} Fixed Correction Cannot Replace Test-time Adaptation
  \item \textbf{Appendix E:} Physics Perturbation Robustness
  \item \textbf{Appendix F:} ACT ManiSkill Results
  \item \textbf{Appendix G:} $\pi_{0.5}$ LIBERO Results
\end{itemize}

\clearpage

\section{Stability Analysis of Theorem~\ref{thm:nominal_main}}
\label{sec:appendix_stability}

This appendix proves Theorem~\ref{thm:nominal_main} from 
Sec.~\ref{sec:adaptive_design}. The proof first establishes asymptotic convergence under the 
nominal setting ($d \equiv 0$), then extends to the bounded-disturbance 
case via \textit{Young}'s inequality.

\subsection{Closed-Loop Error Dynamics}
\label{sec:appendix_dynamics}

We analyze a single execution channel. The scalar nominal model 
($d \equiv 0$) is
\begin{equation}
    J\ddot a = K_p(a_r - a) - K_d \dot a,
    \label{eq:appendix_inner_dynamics}
\end{equation}
where $a$ is the realized action coordinate, $a_r$ is the adjusted 
command sent to the low-level controller, and $J, K_p, K_d > 0$ are 
unknown execution parameters. The policy provides a fixed desired 
command $a_\pi$. Following the main text, define the tracking error 
and sliding variable
\begin{equation}
    e_a = a_\pi - a, 
    \qquad 
    s = \dot a - \dot{\bar a},
\end{equation}
with the auxiliary smooth reference $\bar a$ generated by
\begin{equation}
    \dot y = -K_1 y - K_1 e_a, 
    \qquad 
    \dot{\bar a} = -K_2 y, 
    \qquad 
    \ddot{\bar a} = K_1 K_2(y + e_a).
    \label{eq:appendix_filter}
\end{equation}
For the Lyapunov construction we also introduce the auxiliary positional error
\begin{equation}
    \zeta = a - \bar a, 
    \label{eq:appendix_zeta}
\end{equation}
which directly links the sliding variable $s$ (i.e., $\dot \zeta = s$) to a {positional} quantity.

APEX uses the adaptive command
\begin{equation}
    a_r = \bar a 
    + \hat w \bigl[\hat J \ddot{\bar a} + \alpha e_a\bigr]
    + \hat w_d \dot{\bar a},
    \label{eq:appendix_adaptive_command}
\end{equation}
with online updates
\begin{equation}
    \dot{\hat w} = -\Gamma_w s\bigl[\hat J \ddot{\bar a} + \alpha e_a\bigr],
    \qquad
    \dot{\hat w}_d = -\Gamma_{wd} s \dot{\bar a},
    \qquad
    \dot{\hat J} = -\Gamma_J s \ddot{\bar a}.
    \label{eq:appendix_adaptive_law}
\end{equation}
Eqs.~\eqref{eq:appendix_adaptive_command}--\eqref{eq:appendix_adaptive_law} 
are identical to Eqs.~\eqref{eq:adaptive_ar}--\eqref{eq:adaptive_law} 
in the main text. Define the parameter estimation errors
\begin{equation}
    \tilde w = \hat w - w, 
    \qquad 
    \tilde w_d = \hat w_d - w_d, 
    \qquad 
    \tilde J = \hat J - J,
\end{equation}
where $w = 1/K_p$ and $w_d = K_d/K_p$.
The proof only requires finite initial conditions. In implementation, 
we initialize $\bar a(0) = a(0)$, giving $\zeta(0) = 0$ and reducing 
the initial transient.

Substituting Eq.~\eqref{eq:appendix_adaptive_command} into 
Eq.~\eqref{eq:appendix_inner_dynamics} and expanding 
$\hat w = w + \tilde w$, $\hat w_d = w_d + \tilde w_d$, 
$\hat J = J + \tilde J$, the $s$-dynamics become
\begin{equation}
\begin{aligned}
    J \dot s
    &= -K_d s - K_p \zeta + \alpha e_a \\
    &\quad + K_p \tilde w \bigl[\hat J \ddot{\bar a} + \alpha e_a\bigr]
    + K_p \tilde w_d \dot{\bar a}
    + \tilde J \ddot{\bar a}.
\end{aligned}
\label{eq:appendix_s_dyn}
\end{equation}
The parametric uncertainty enters in three places (one per adaptive 
coefficient), each multiplied by a measurable regressor. Since $a_\pi$ 
is fixed, $\dot e_a = -\dot a$, and the filter-state dynamics reduce to
\begin{equation}
    \dot e_a = -s + K_2 y,
    \label{eq:appendix_ea_dyn}
\end{equation}
\begin{equation}
    \dot y = -K_1 y - K_1 e_a.
    \label{eq:appendix_y_dyn}
\end{equation}

\subsection{Proof of Theorem~\ref{thm:nominal_main}}
\label{sec:appendix_proof}

\paragraph{Lyapunov candidate.}
We choose
\begin{equation}
\begin{aligned}
    V 
    &= \underbrace{\tfrac{1}{2} J s^2}_{\text{sliding}}
    + \underbrace{\tfrac{1}{2} K_p \zeta^2}_{\text{position}}
    + \underbrace{\tfrac{1}{2} \alpha e_a^2 
    + \tfrac{1}{2} \alpha \tfrac{K_2}{K_1} y^2}_{\text{filter subsystem}} \\
    &\quad + \underbrace{\tfrac{K_p}{2\Gamma_w}\tilde w^2 
    + \tfrac{K_p}{2\Gamma_{wd}}\tilde w_d^2 
    + \tfrac{1}{2\Gamma_J}\tilde J^2}_{\text{parameter errors}}.
\end{aligned}
\label{eq:appendix_lyapunov}
\end{equation}
The filter-subsystem terms are weighted to cancel the cross-coupling 
between $e_a$ and $y$ in Eqs.~\eqref{eq:appendix_ea_dyn}--\eqref{eq:appendix_y_dyn} 
in the derivative. The parameter-error weights match the inverse learning 
rates so that the adaptive law in 
Eq.~\eqref{eq:appendix_adaptive_law} cancels the parametric terms in 
$\dot s$. Since $J, K_p, \alpha, K_1, K_2, \Gamma_w, \Gamma_{wd}, 
\Gamma_J > 0$, $V$ is positive definite in the error variables.

\paragraph{Derivative along trajectories.}
Differentiating Eq.~\eqref{eq:appendix_lyapunov} along 
Eqs.~\eqref{eq:appendix_s_dyn}--\eqref{eq:appendix_y_dyn}, the state 
cross terms cancel
\begin{equation}
    -K_p \zeta s + K_p \zeta s = 0,
    \qquad
    +\alpha s e_a + \alpha e_a(-s) = 0,
\end{equation}
\begin{equation}
    \alpha K_2 e_a y - \alpha K_2 y\, e_a = 0.
\end{equation}
The parameter-error products are cancelled by the adaptive 
law~\eqref{eq:appendix_adaptive_law}
\begin{equation}
\begin{aligned}
&K_p s \tilde w \bigl(\hat J \ddot{\bar a} + \alpha e_a\bigr)
+ \tfrac{K_p}{\Gamma_w} \tilde w \dot{\hat w} = 0, \\
&K_p s \tilde w_d \dot{\bar a}
+ \tfrac{K_p}{\Gamma_{wd}} \tilde w_d \dot{\hat w}_d = 0, \\
&s \tilde J \ddot{\bar a}
+ \tfrac{1}{\Gamma_J} \tilde J \dot{\hat J} = 0.
\end{aligned}
\end{equation}
The remaining terms yield
\begin{equation}
    \dot V = -K_d s^2 - \alpha K_2 y^2 \le 0.
    \label{eq:appendix_Vdot_final}
\end{equation}

\paragraph{Boundedness.}
Since $\dot V \le 0$, $V$ is non-increasing, and all error variables
\begin{equation}
    s,\, \zeta,\, e_a,\, y,\, \tilde w,\, \tilde w_d,\, \tilde J
\end{equation}
are bounded. With $a_\pi$ fixed, bounded $e_a = a_\pi - a$ implies 
bounded $a$. Bounded $\zeta = a - \bar a$ then implies bounded 
$\bar a$. From Eq.~\eqref{eq:appendix_filter}, bounded $y$ and 
$e_a$ imply bounded $\dot{\bar a}, \ddot{\bar a}$. Boundedness of 
$\hat w, \hat w_d, \hat J$ follows from bounded 
$\tilde w, \tilde w_d, \tilde J$. Finally, 
Eq.~\eqref{eq:appendix_adaptive_command} gives bounded $a_r$, and 
$\dot a = s + \dot{\bar a}$ is bounded. Thus all closed-loop signals 
remain bounded.

\paragraph{Asymptotic convergence via \textit{Barbalat}'s lemma.}
Bounded right-hand sides in Eqs.~\eqref{eq:appendix_s_dyn} and 
\eqref{eq:appendix_y_dyn} make $\dot s, \dot y$ bounded, so $s$ and 
$y$ are uniformly continuous. Integrating 
Eq.~\eqref{eq:appendix_Vdot_final} from $0$ to $\infty$,
\begin{equation}
    K_d \int_0^\infty s^2(\tau)\, d\tau 
    + \alpha K_2 \int_0^\infty y^2(\tau)\, d\tau 
    \le V(0) - V(\infty) < \infty,
\end{equation}
gives $s, y \in L_2$. Combined with uniform continuity, 
Barbalat's lemma~\citep{khalil2002nonlinear} yields
\begin{equation}
    \lim_{t\to\infty} s(t) = 0, 
    \qquad 
    \lim_{t\to\infty} y(t) = 0.
\end{equation}

\paragraph{Convergence of the tracking error.}
The $(e_a, y)$ subsystem from 
Eqs.~\eqref{eq:appendix_ea_dyn}--\eqref{eq:appendix_y_dyn} is
\begin{equation}
\frac{d}{dt}
\begin{bmatrix}e_a \\ y\end{bmatrix} 
= 
\begin{bmatrix}0 & K_2 \\ -K_1 & -K_1\end{bmatrix}
\begin{bmatrix}e_a \\ y\end{bmatrix} 
+ 
\begin{bmatrix}-1 \\ 0\end{bmatrix} s.
\label{eq:appendix_linear_subsystem}
\end{equation}
The system matrix has trace $-K_1 < 0$ and determinant $K_1 K_2 > 0$, 
hence is Hurwitz for any $K_1, K_2 > 0$. Since $s(t) \to 0$ acts as a 
vanishing input and the state is bounded, standard input-to-state 
arguments for LTI systems~\citep{khalil2002nonlinear} give
\begin{equation}
    \lim_{t\to\infty} e_a(t) = 0,
\end{equation}
which together with $s(t) \to 0$ completes the nominal proof.

\paragraph{Extension to bounded disturbances.}
When the lumped disturbance $d(t)$ in Eq.~\eqref{eq:inner_dynamics} is 
nonzero but bounded, $|d(t)| \le \bar d$ with $\bar d > 0$, the 
$s$-dynamics in Eq.~\eqref{eq:appendix_s_dyn} acquire an additional 
term $d(t)$ on the right-hand side, and the \textit{Lyapunov} derivative 
becomes
\begin{equation}
    \dot V = -K_d s^2 - \alpha K_2 y^2 + s\, d(t).
\end{equation}
Applying \textit{Young}'s inequality, 
$s\, d(t) \le \tfrac{K_d}{2} s^2 + \tfrac{1}{2 K_d} d^2(t) 
\le \tfrac{K_d}{2} s^2 + \tfrac{\bar d^2}{2 K_d}$, yields
\begin{equation}
    \dot V \le -\tfrac{K_d}{2} s^2 - \alpha K_2 y^2 
    + \tfrac{\bar d^2}{2 K_d}.
    \label{eq:appendix_Vdot_disturbance}
\end{equation}
Standard input-to-state stability 
arguments~\citep{khalil2002nonlinear} then ensure that all closed-loop 
signals remain bounded and the sliding variable and tracking error 
converge to a residual set whose size scales linearly with the 
disturbance bound:
\begin{equation}
    \limsup_{t\to\infty} |s(t)| \le c_1\, \bar d,
    \qquad
    \limsup_{t\to\infty} |e_a(t)| \le c_2\, \bar d,
\end{equation}
for constants $c_1, c_2 > 0$ depending on the design gains. As 
$\bar d \to 0$, the residual set shrinks to the origin, recovering 
the nominal result $s(t), e_a(t) \to 0$. \qed

Theorem~\ref{thm:nominal_main} assumes a fixed $a_\pi$, while a 
deployed policy issues a sequence of commands. At each policy update, 
the physical states ($a, \dot a$), filter state $y$, reference 
$\bar a$, command $a_r$, and adaptive parameters $\hat w, \hat w_d, 
\hat J$ remain finite. Only $e_a = a_\pi - a$ jumps when $a_\pi$ 
changes. Provided the policy outputs are bounded, each held-command 
interval begins from a finite initial condition, and 
Theorem~\ref{thm:nominal_main} applies interval-wise. This yields 
stabilization and error attenuation within each interval rather than 
global convergence to an arbitrary time-varying sequence. The 
empirical evaluation in Sec.~\ref{sec:exp} measures the resulting 
finite-time error reduction over full rollouts.


\section{Experimental Setup}
\label{sec:setup}
\FloatBarrier


This appendix provides the implementation details behind the experimental results in the main paper. For each experiment family, we specify the simulator or hardware platform, task suite, policy class, low-level execution interface, and evaluation protocol used for comparing direct execution with APEX.

\subsection{Simulation Environments}

Fig.~\ref{fig:sim_tasks} summarizes the simulated environments, which include four ManiSkill tasks and three LIBERO Spatial tasks. In the LIBERO experiments, we evaluate the $\pi_{0.5}$ VLA, whose end-effector delta actions are first converted into joint-space commands through inverse kinematics before being passed to the APEX correction layer.

\renewcommand*{\thefigure}{S1}
\begin{figure}[H]
    \centering
    \includegraphics[width=\linewidth]{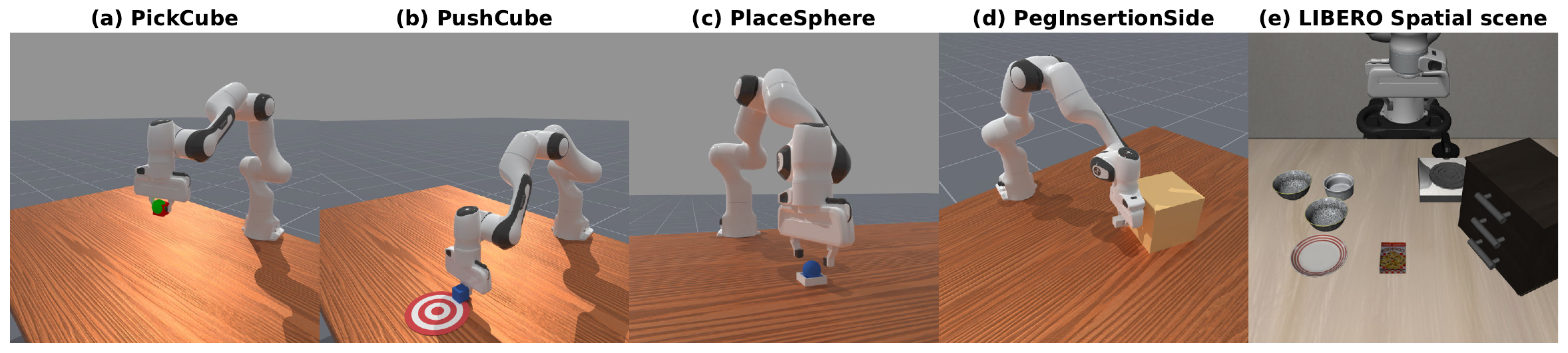}
    \caption{Simulation tasks and scenes used in our experiments.}
    \label{fig:sim_tasks}
\end{figure}

\subsection{Real-robot Environments}
\paragraph{Hardware platform.}
All real-robot experiments are conducted on a 6-DoF UR5 manipulator equipped with a Robotiq Hand-E adaptive parallel-jaw gripper. Visual observations are captured by two Intel RealSense D435i (eye-in-hand) \& D455 (eye-to-hand) cameras. Policy inference and execution-side correction run on a workstation with an NVIDIA RTX 3090.

\paragraph{Control stack.}
The DP and FM policies output end-effector delta actions, which are mapped to joint-space commands via inverse kinematics before execution-side correction. As a commercial platform, the UR5 exposes only a position-command interface and does not expose or allow modification of its internal controller gains. The execution-side terms of APEX are therefore estimated online from the measured tracking error rather than from known controller parameters. 

\paragraph{Tasks \& dataset.}
We evaluate three manipulation tasks of increasing precision demand: PushBox (planar pushing of a box to a goal region), PickCube (grasping a cube and placing it in a box), and PegInsertion (inserting a peg into a hole whose radius exceeds the peg
radius by 3\,mm). For each task, we collect 60 demonstrations via teleoperation at 30\,Hz, recording RGB images and proprioceptive joint states. The tasks differ markedly in how strongly their outcome depends on execution accuracy: PushBox succeeds under
coarse motion, whereas PegInsertion is dominated by millimeter-level alignment and
is therefore the most sensitive to residual tracking error.


\paragraph{Evaluation protocol.}
We train one DP~\citep{chi2023diffusion}
and one FM~\citep{flowmatching} policy per task using the best validation epoch.
Each task-policy pair is evaluated over 20 trials. Each episode is capped at 200 inference steps, beyond which the trial is counted as a failure.
The Baseline corresponds to the original policy executed by the joint-position servo
without APEX. All other factors (controller, initial-pose distribution,
success criterion) are held identical, so that APEX is the only variable.


\subsection{Evaluation Metrics}
We evaluate each method using three metrics: success rate, action-tracking error, and paired outcome changes. The tracking error measures the discrepancy between the action reference and the realized action. The reference is the expert trajectory in replay experiments and the policy output in learned-policy rollouts. For paired outcome changes, we count both cases where APEX turns a failed baseline rollout into a success and cases where it turns a successful baseline rollout into a failure. ManiSkill results report joint-space tracking error in mrad, whereas LIBERO results report end-effector tracking error in millimeters; we therefore summarize them separately. Each ManiSkill learned-policy setting is evaluated with 100 paired episodes, and each LIBERO $\pi_{0.5}$ setting with 20 paired episodes. When computing ManiSkill tracking-error summaries, we exclude rare anomalous steps where the policy command, executed command, or observed joint state exceeds 4 rad. This removes obvious policy-output or simulator failures from the controller-tracking metric, while leaving terminal success evaluation unchanged.

\section{Implementation Details of APEX}
\label{sec:appendix_impl}

The dynamic filter in Sec.~\ref{sec:method} is introduced for theoretical completeness: it provides smooth reference derivatives and yields a clean stability analysis. In the experiments, we use a simpler discrete-time realization that preserves the same adapter role but estimates derivatives by finite differences and uses a re-parameterized set of bounded adaptive coefficients. We found this implementation robust across the replay, learned-policy, and controller-sweep settings.

At each control step, APEX receives the policy-level action reference $a_\pi$ and the realized action state $(a,\dot a)$. Using the tracking error $e_a=a_\pi-a$, the adapter sends the corrected command
\begin{equation}
    a_r
    =
    a_\pi
    + \hat\alpha e_a
    + \hat\theta_v \dot a
    + \hat\theta_a \ddot a_\pi^{\rm fd},
    \label{eq:impl_ar}
\end{equation}
where $\ddot a_\pi^{\rm fd}$ is a finite-difference estimate from recent policy references. The adaptive coefficients $\hat\alpha$, $\hat\theta_v$, and $\hat\theta_a$ correspond to position-, velocity-, and acceleration-level correction terms. The coefficients are updated online from tracking feedback, following the same adaptation principle as Sec.~\ref{sec:method}.

\subsection{Hyperparameters}
\label{sec:appendix_hyperparams}

Tab.~\ref{tab:apex_hyperparams} summarizes the main APEX hyperparameters used in replay and ACT controller-grid experiments. We use a single adaptive preset across these settings. All correction coefficients are initialized from zero, so the gains come from online adaptation rather than a tuned warm start. The update gains control the adaptation rates of the position-, velocity-, and acceleration-level terms, $\lambda_s$ denotes the sliding-variable bandwidth, and $c_{\max}=0.10$ denotes the maximum absolute command compensation used to clip the final correction for improved robustness.

\renewcommand*{\thetable}{S1}
\begin{table}[t]
\centering
\small
\setlength{\tabcolsep}{8pt}
\renewcommand{\arraystretch}{1.12}
\caption{Main APEX hyperparameters.}
\label{tab:apex_hyperparams}
\begin{tabular}{lc}
\toprule
\textbf{Parameter} & \textbf{Value} \\
\midrule
$\hat\alpha_0$ & $0$ \\
$\hat\theta_{v,0}$ & $0$ \\
$\hat\theta_{a,0}$ & $0$ \\
$\gamma_\alpha$ & $30$ \\
$\gamma_v$ & $0.1$ \\
$\gamma_a$ & $0.1$ \\
$\lambda_s$ & $30.0$ \\
$c_{\max}$ & $0.10$ \\
\bottomrule
\end{tabular}
\end{table}

\section{Fixed Correction Cannot Replace Test-time Adaptation}
\label{app:fixed_vs_adaptive_sweep}

This section examines whether the gains of APEX can be attributed to online adaptation rather than to a fixed correction preset. If the correction with constant coefficients were sufficient, a single set of correction coefficients should remain effective when the same reference trajectory is executed under different low-level controller gains.

We conduct this ablation on PegInsertionSide-v1 replay, where the action reference is an expert motion-planning trajectory, and policy prediction error is removed. We vary the test-time proportional gain $K_p \in\{500,1000,1500\}$ while keeping $K_d=100$ fixed. For the fixed-correction variant, all adaptive updates are disabled and the coefficients are swept over
\[
\hat\alpha\in\{0,0.2,0.4,0.6\},\quad
\hat\theta_v\in\{0,0.05,0.1,0.2\},\quad
\hat\theta_a\in\{0,0.01,0.02,0.05\}.
\]
This yields 64 fixed configurations, each evaluated with 30 paired replay episodes. Adaptive APEX is evaluated from the neutral initialization $\hat\alpha_0=\hat\theta_{v,0}=\hat\theta_{a,0}=0$, using $\gamma_\alpha=30$, $\gamma_a=0.1$, and $\gamma_v=0.05$. Thus, any improvement of adaptive APEX must arise from online parameter updates during execution rather than from a tuned warm start.

The sweep shows that fixed correction is only effective within a limited region of the controller-gain space. The best fixed configuration achieves a mean improvement of $+16.7$ pp across the evaluated $K_p$ values, whereas adaptive APEX reaches $+23.3$ pp despite starting from zero correction. These results support the claim that APEX benefits from adapting to the current execution dynamics online, rather than applying a transferable fixed correction.

\FloatBarrier
\section{Physics Perturbation Robustness}
\label{app:physics-perturbation}

We conduct this study on PegInsertionSide-v1 replay with $K_p=700$, a mildly under-stiffened controller setting where direct replay is possible but unreliable. The expert reference, controller gains, and rollout seeds are kept fixed, and only the joint-level physics parameters are varied. We consider two perturbation families. First, we vary the dimensionless joint Coulomb friction coefficient over $\{0,5,10,15,20,25,30\}$ while keeping armature fixed; this coefficient is the simulator's dry-friction parameter, which bounds the friction torque opposing joint motion in proportion to the joint load, so a larger value imposes proportionally stronger resistance. Second, we vary rotor armature, which changes motor-side inertia, over $\{0,0.05,0.10,0.15,0.20,0.25,0.30\}$ kg$\cdot$m$^2$ while keeping friction fixed. The armature sweep covers and extends beyond the typical Franka range of 0.05--0.15 kg$\cdot$m$^2$. For each setting, we compare direct replay and APEX using paired rollouts, terminal success, and action-tracking error.

\renewcommand*{\thetable}{S2}
\begin{table}[H]
\centering
\footnotesize
\setlength{\tabcolsep}{5.5pt}
\renewcommand{\arraystretch}{0.95}
\caption{Complete fixed-correction sweep on PegInsertionSide replay. Each entry reports the mean success-rate change over $K_p\in\{500,1000,1500\}$ in percentage points. All fixed configurations disable online adaptation and vary only the fixed values of the position-, velocity-, and acceleration-level correction coefficients.}
\label{tab:fixed_sweep_all_peg}
\vspace{1mm}
\begin{tabular}{ccrrrr}
\toprule
$\hat\alpha$ & $\hat\theta_v$ & $\hat\theta_a=0$ & $\hat\theta_a=0.01$ & $\hat\theta_a=0.02$ & $\hat\theta_a=0.05$ \\
\midrule
0.0 & 0.00 & $0.0$ & $-10.0$ & $-11.1$ & $-22.2$ \\
0.0 & 0.05 & $-13.3$ & $-13.3$ & $-13.3$ & $-11.1$ \\
0.0 & 0.10 & $-37.8$ & $-35.6$ & $-42.2$ & $-52.2$ \\
0.0 & 0.20 & $-58.9$ & $-56.7$ & $-57.8$ & $-60.0$ \\
\midrule
0.2 & 0.00 & $0.0$ & $0.0$ & $0.0$ & $-6.7$ \\
0.2 & 0.05 & $0.0$ & $-2.2$ & $2.2$ & $-3.3$ \\
0.2 & 0.10 & $-32.2$ & $-33.3$ & $-33.3$ & $-33.3$ \\
0.2 & 0.20 & $-62.2$ & $-60.0$ & $-62.2$ & $-60.0$ \\
\midrule
0.4 & 0.00 & $10.0$ & $5.6$ & $5.6$ & $1.1$ \\
0.4 & 0.05 & $5.6$ & $11.1$ & $7.8$ & $8.9$ \\
0.4 & 0.10 & $-33.4$ & $-32.2$ & $-31.1$ & $-31.1$ \\
0.4 & 0.20 & $-57.8$ & $-57.8$ & $-60.0$ & $-56.7$ \\
\midrule
0.6 & 0.00 & $15.6$ & $\mathbf{16.7}$ & $13.4$ & $6.7$ \\
0.6 & 0.05 & $13.3$ & $13.3$ & $12.2$ & $16.7$ \\
0.6 & 0.10 & $-35.6$ & $-32.2$ & $-32.2$ & $-30.0$ \\
0.6 & 0.20 & $-57.8$ & $-60.0$ & $-55.6$ & $-57.8$ \\
\bottomrule
\end{tabular}
\end{table}

\renewcommand*{\thefigure}{S2}
\begin{figure}[H]
    \centering
    \includegraphics[width=0.84\linewidth]{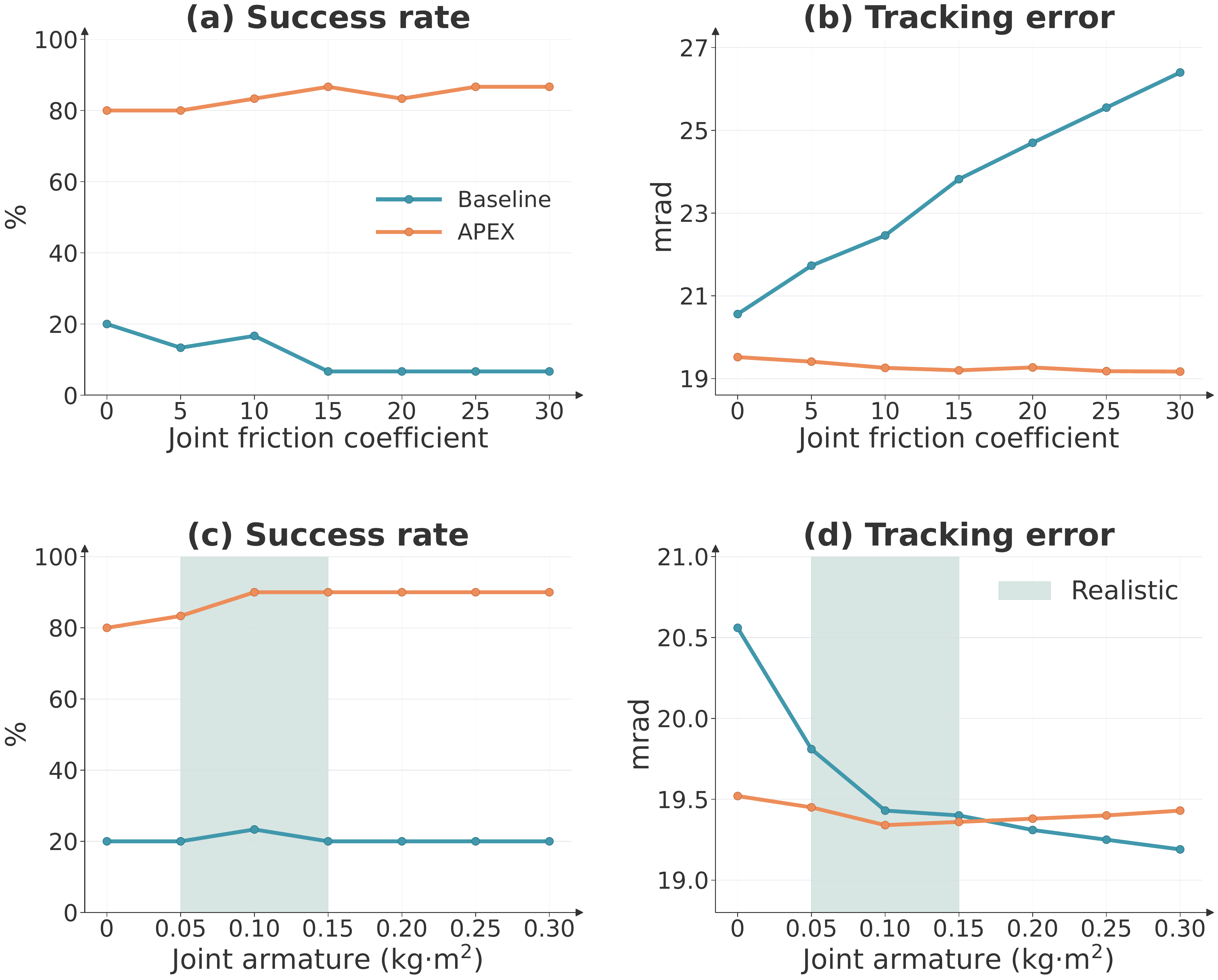}
    \caption{
    Physics perturbation robustness on PegInsertionSide-v1 replay.
    Panels (a,b) vary the joint Coulomb friction coefficient; panels (c,d) vary the rotor armature.
    Baseline is replay without APEX, and the shaded region marks the realistic
    Franka armature range.
    }

    \label{fig:physics-perturbation}
\end{figure}

Under increasing friction, direct replay degrades substantially. As the friction coefficient increases from 0 to 30, its success rate drops from 20\% to 7\%, and its tracking error increases from 20.6 to 26.4 mrad. In contrast, APEX maintains 80--87\% success rate, with tracking error remaining between 19.2 and 19.5 mrad. Consequently, the success gain of APEX increases from +60 to +80 pp as friction becomes more severe.

The armature perturbation exhibits a different failure mode. Mean tracking error changes only mildly for both direct replay and APEX, yet the terminal outcomes remain clearly separated. Direct replay achieves only 20--23\% success across the sweep, whereas APEX maintains 80--90\% success, including all settings within the typical Franka armature range.

These results indicate that APEX is not merely tuned to the default simulator dynamics. When friction increases, APEX prevents the tracking-error degradation observed under direct replay. When motor-side inertia changes, APEX preserves high task success even though the mean tracking error is less sensitive to this parameter. Within the PegInsertion replay setting, APEX therefore remains robust to the tested perturbations in joint friction and rotor armature.

\section{ACT ManiSkill Results}
\label{app:act_full_grid}
\FloatBarrier

\renewcommand*{\thefigure}{S3}
\begin{figure}[H]
  \centering
  \includegraphics[width=0.84\linewidth]{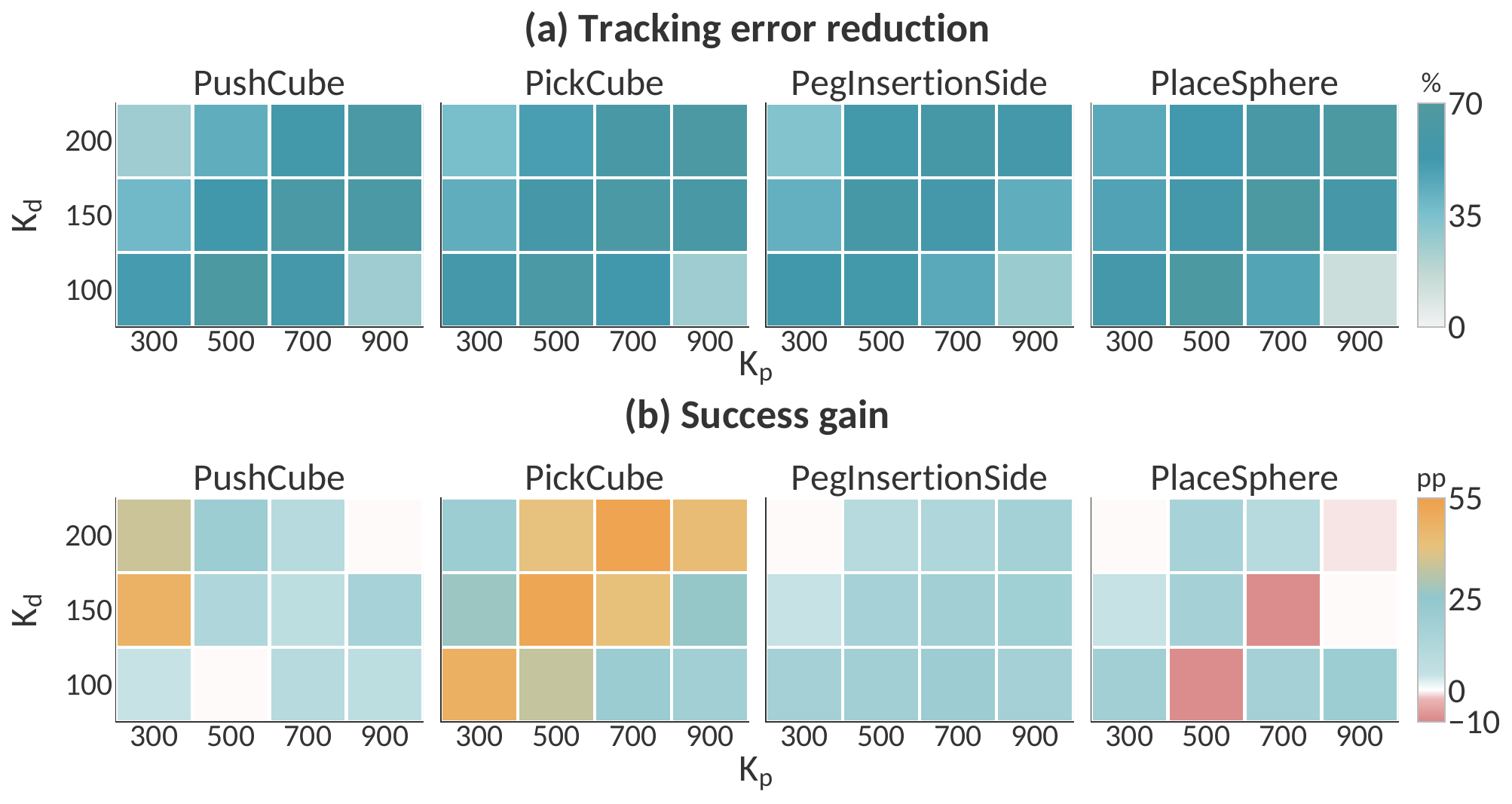}
  \caption{Full ACT performance grid. \textbf{(a)} Tracking-error reduction in percentage.
  \textbf{(b)} Success-rate gain in percentage points. Each heatmap sweeps $K_p$
  and $K_d$ for one task.}
  \label{fig:act_full_grid}
\end{figure}

Fig.~\ref{fig:act_full_grid} provides the complete ACT controller-gain sweep corresponding to Sec.~\ref{sec:act_flagship}. For each task, the policy is fixed and only the low-level PD gains are varied. Each setting is evaluated with 100 paired rollouts using matched initial seeds for direct execution and APEX. The heatmaps sweep $K_p$ along the horizontal axis and $K_d$ along the vertical axis, reporting tracking-error reduction in the top row and success-rate gain in the bottom row. Across the 48 task-controller settings, APEX improves the mean success rate from 27.4\% to 40.6\%. It increases success in 40 settings, leaves 5 unchanged, and decreases success in only 3, showing that the correction is broadly effective across controller-gain mismatch rather than tuned to a single configuration.

\section{\texorpdfstring{$\pi_{0.5}$}{pi0.5} LIBERO Results}
\label{app:pi05_libero_kp}
\FloatBarrier

Fig.~\ref{fig:pi05_libero_kp_grid} provides the complete $\pi_{0.5}$ controller-gain sweep corresponding to Sec.~\ref{sec:cross_policy}. For each LIBERO Spatial task, the policy is fixed and only the proportional-gain scale of the low-level controller is varied. Each setting is evaluated with 20 paired rollouts using matched initial seeds for direct execution and APEX. The curves sweep the $K_p$ scale, reporting task success in the top row and end-effector tracking error in the bottom row. Here, tracking error is measured as the positional deviation between the target and realized end-effector positions during policy rollout.

The evaluated LIBERO Spatial tasks are listed in Tab.~\ref{tab:pi05_libero_tasks}. Across the 12 task-controller settings, APEX improves the mean success rate from 72.1\% to 97.9\%. It increases success in 11 settings and leaves 1 unchanged, indicating that the execution-side correction remains effective for a VLA policy with end-effector delta actions and inverse-kinematics-based execution.



\renewcommand*{\thetable}{S3}
\begin{table}[htbp]
    \centering
    \small
    \caption{$\pi_{0.5}$ LIBERO Spatial tasks used in the controller-gain sweep.}
    \label{tab:pi05_libero_tasks}
    \begin{tabularx}{\linewidth}{c>{\raggedright\arraybackslash}X}
        \toprule
        Task & Language instruction \\
        \midrule
        1 & Pick up the black bowl between the plate and the ramekin and place it on the plate. \\
        2 & Pick up the black bowl next to the ramekin and place it on the plate. \\
        3 & Pick up the black bowl from the table center and place it on the plate. \\
        \bottomrule
    \end{tabularx}
\end{table}

\renewcommand*{\thefigure}{S4}
\begin{figure}[H]
    \centering
    \includegraphics[width=0.84\linewidth]{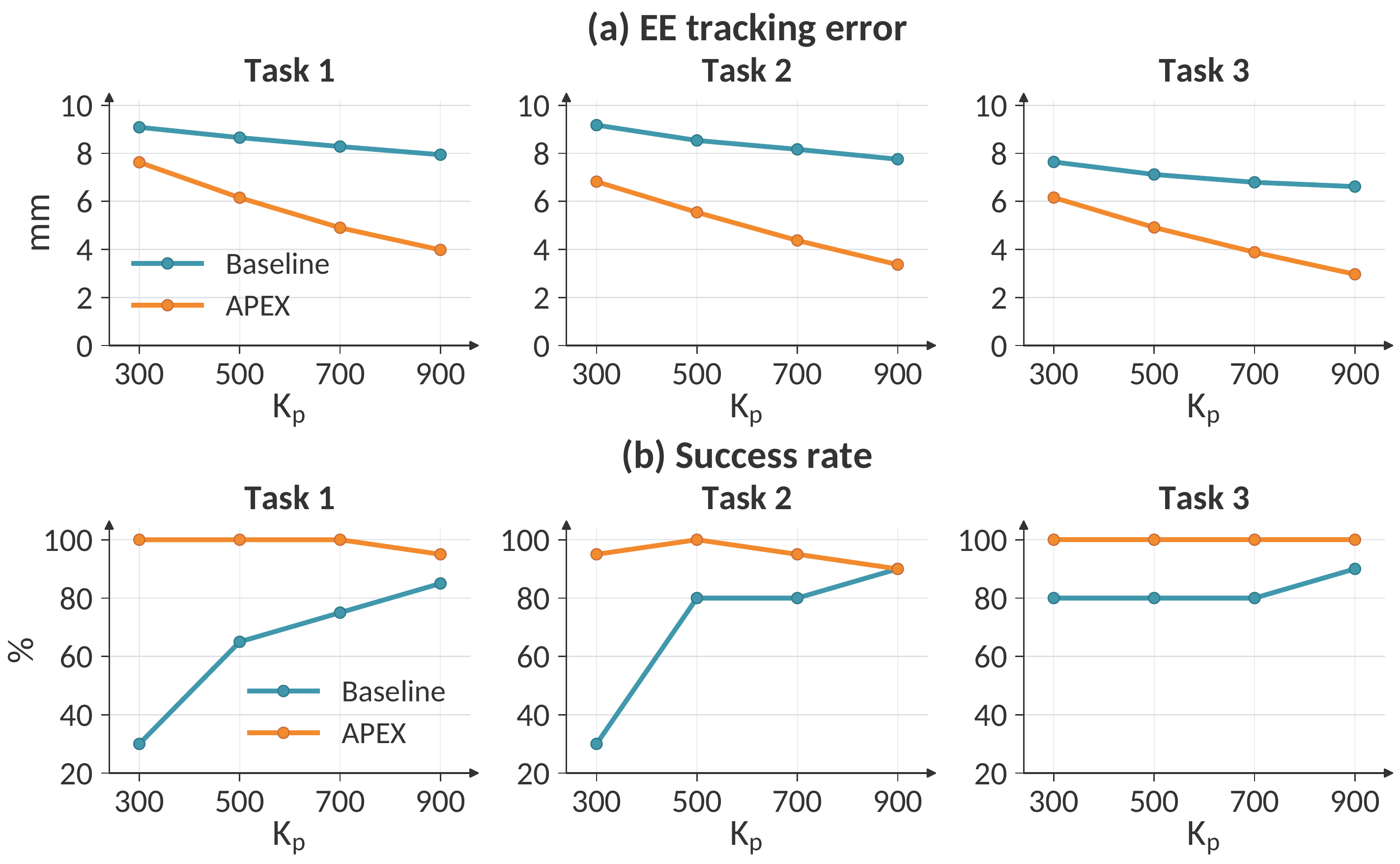}
    \vspace{-10pt}
    \caption{Full $\pi_{0.5}$ performance grid on LIBERO Spatial tasks. \textbf{(a)}
    End-effector tracking error. \textbf{(b)} Task success rate. Each column
    corresponds to one task, and each curve sweeps the $K_p$ scale under the
    same fixed policy. }
    \label{fig:pi05_libero_kp_grid}
    \vspace{-5pt}
\end{figure}

\end{document}